\pdfoutput=1
\documentclass{article}

\usepackage{iclr2027_conference,times}
\iclrfinalcopy

\usepackage{amsmath,amsfonts,bm}

\def\eqref#1{equation~\ref{#1}}

\def\1{\bm{1}}

\DeclareMathAlphabet{\mathsfit}{\encodingdefault}{\sfdefault}{m}{sl}
\SetMathAlphabet{\mathsfit}{bold}{\encodingdefault}{\sfdefault}{bx}{n}

\usepackage{amsmath}
\usepackage{amssymb}
\usepackage{amsthm}
\usepackage{booktabs}
\usepackage{graphicx}
\usepackage{multirow}
\usepackage{listings}
\usepackage{xcolor}
\usepackage{url}
\usepackage{hyperref}
\hypersetup{colorlinks=true,linkcolor=blue,citecolor=blue,urlcolor=blue}

\definecolor{codebg}{rgb}{0.96,0.96,0.96}
\newtheorem{proposition}{Proposition}

\title{CASI: Credit Assignment via Set Intervention\\ for Agentic Retrieval}

\author{
  Xinyi Cheng \quad Tianfu Zhang \quad Jiwen Li \quad Xuesen Zhang \\
  \textbf{Shihao Yu \quad Pengfei Zhu \quad Renhuan \quad Heda Wang}\thanks{Corresponding author.} \\
  Xiaohongshu Inc. \\
  \texttt{\{chengxinyi1, zhangtianfu, liuwei01, zhangxuesen, }\\ \texttt{yushihao, zhupengfei, renhuan, zengshushu\}@xiaohongshu.com}
}

\begin{document}

\maketitle

\lhead{Preprint}

\begin{abstract}
Large language models increasingly generate multiple retrieval actions jointly,
while downstream systems often provide only one reward for the resulting set.
Broadcasting this reward provides a valid score-function gradient estimator for
the expected set-level objective, but does not determine a coalition-dependent
within-set attribution. We introduce \textbf{CASI} (Credit Assignment via Set
Intervention for Agentic Retrieval), which constructs the missing counterfactual information through
snapshot-grounded set intervention: generate and retrieve once, hold replayable
snapshot components fixed, and evaluate alternative subsets under common
conditions. Shapley aggregation converts these marginals into within-set credit,
which bounded zero-sum advantage shaping incorporates into policy optimization.
For $m\ge4$, singleton, leave-one-out, and full-set endpoints do not generally
determine the Shapley allocation; under additive set values, extra coalition
evaluations are redundant. In a five-action sponsored-search setting,
exhaustive auction replay over 1,200 requests reveals substantial coalition
dependence. Under the training-time evaluator, CASI achieves the highest median
cosine similarity to a replay-defined Shapley reference among evaluated credit
rules; under a sensitivity backend, CASI does not outperform Singleton. Across
five independent training runs, CASI-trained policies achieve higher held-out
auction-replay set value than Singleton- or Broadcast-trained policies. In a
randomized seven-day production A/B test on 10\% of traffic, the CASI-trained
policy yields an observed \textbf{3.51\%} lift in the primary advertiser-value
metric over production control.
\end{abstract}
\section{Introduction}
\label{sec:intro}

LLM-based retrieval systems increasingly generate multiple retrieval actions
for a single input, for example by decomposing a question or producing parallel
rewrites that cover different facets of an information need~\citep{ragr1,parallelsearch,parallelr1}.
Post-training on downstream outcomes directly optimizes what the system
ultimately delivers rather than surface similarity alone~\citep{searchr1}.
This challenge is especially pronounced in sponsored search, where several
rewrites retrieve candidate ads whose pools are merged before downstream
ranking and auction. An action that is useful alone may become redundant, or
compete with candidates from other rewrites, once the pools are combined. In
our setting, individual actions do not receive separately observed downstream
outcomes; only the realized action set receives an observed reward. This raises
a central question: \emph{when multiple retrieval actions jointly determine one
outcome, what information is needed to construct context-dependent attribution
within the realized set?}

A rollout-level reward provides a valid score-function estimator for the
gradient of the expected set-level objective~\citep{deepseekmath}. The
difficulty is that the realized set reward does not provide the
alternative-subset values needed for coalition-based attribution. For an
action $q_i$ and $S\subseteq A\setminus\{q_i\}$, the marginal
$R(S\cup\{q_i\})-R(S)$ can vary with $S$. Singleton and leave-one-out capture
the empty-context and full-complement extremes, leaving intermediate coalition
contexts unobserved. For $m\ge4$, these endpoint observations, even together
with the full-set value, do not generally determine the Shapley allocation.
The challenge is therefore to construct comparable cross-context information
within each realized action set.

We introduce \textbf{CASI} (\emph{Credit Assignment via Set Intervention for Agentic Retrieval}),
which constructs this information through \emph{snapshot-grounded set
intervention}. For each realized action set, CASI generates and retrieves once,
caches per-action retrieval results, holds replayable snapshot components
fixed, and evaluates alternative subsets under common conditions. Shapley
weighting converts the resulting contextual marginals into within-set credit,
and bounded zero-sum advantage shaping incorporates their relative credit into
the rollout-level optimizer. Reusing cached retrieval outputs keeps exhaustive
coalition scoring practical in our five-action setting without repeated
generation, retrieval, or auction replay during RL.

Empirically, exhaustive auction replay over $1{,}200$ requests reveals
substantial coalition dependence. Under the training-time evaluator, CASI
attains the largest median cosine similarity to the replay-defined Shapley
reference among the compared credit rules: $0.424$ versus $0.388$ for
Singleton, a difference in median cosine of $0.036$ (95\% CI
$[0.004,0.068]$). Under the sensitivity backend, CASI does not outperform
Singleton on this metric. Across five independent training runs, the
CASI-trained policy achieves a held-out auction-replay set value of
$66.70\pm3.40$, compared with $60.32$ for the SFT initialization,
$55.00\pm4.00$ for Singleton, and $43.30\pm4.70$ for Broadcast. Finally, in a
randomized seven-day production A/B test on $10\%$ of traffic, we observe a
\textbf{3.51\%} lift in the primary advertiser-value metric over production
control.

Our contributions are threefold. First, we formalize the information gap
between valid set-level optimization and coalition-based within-set attribution
for jointly generated retrieval actions. We show that, for $m\ge4$, singleton,
leave-one-out, and full-set endpoint observations do not generally determine
the Shapley allocation, while under additive set values additional coalition
evaluations are redundant. Second, we introduce snapshot-grounded set
intervention to construct comparable coalition evaluations within a realized
action set while holding replayable snapshot components fixed. Shapley
weighting aggregates these evaluations into action-level credit, and bounded
zero-sum advantage shaping incorporates the relative credit into policy
optimization. Third, we evaluate CASI through coalition structure, credit-rule
alignment with the replay-defined Shapley reference, held-out policy value
across independent training runs, and randomized production deployment.
\section{Related Work}
\label{sec:related_work}

\subsection{LLM Search and Outcome-Based Post-Training}

LLMs increasingly learn retrieval actions from downstream outcomes rather than
surface similarity. Search-R1~\citep{searchr1} trains multi-turn search with
reinforcement learning, while RAG-R1~\citep{ragr1} and parallel-search
methods~\citep{parallelsearch} extend this paradigm to multiple queries or
parallel decomposition. Outcome-based post-training typically assigns
supervision at the response or trajectory level: GRPO~\citep{deepseekmath}
derives a scalar group-relative advantage for each sampled response, whereas
process supervision provides finer feedback along a reasoning
trajectory~\citep{lightman2023verify}. These methods refine where supervision
is assigned across responses or trajectories. Our setting concerns a different
axis: several retrieval actions are generated jointly within one response and
share one downstream outcome, raising the question of what counterfactual
information is needed for context-dependent attribution within the realized
action set.

\subsection{Credit Assignment under Aggregate Feedback}

Shared-reward credit assignment is well studied in cooperative reinforcement
learning. COMA constructs an agent-specific counterfactual baseline by
marginalizing one agent's action while holding others
fixed~\citep{coma}. Standard difference rewards instead remove or replace one
agent~\citep{difference_rewards}. Value-decomposition methods impose structure
on joint value, additively in VDN~\citep{vdn} and monotonically in
QMIX~\citep{qmix}. Shapley-based methods aggregate marginal contributions
across coalitions and have also been used for policy
learning~\citep{shapley_counterfactual,shapley_q}. Slate off-policy evaluation
provides another structural precedent, recovering per-slot components from
slate-level feedback under a specified model~\citep{slate_ope}.

These works provide set-to-element decomposition and Shapley-style allocation
machinery. The methodological gap in our setting is the counterfactual object
on which cross-context within-set credit can be computed. Slate decomposition
relies on an additive conditional-mean reward model, whereas retrieved
candidate pools in our setting are merged and jointly processed by a nonlinear,
capacity-limited downstream system. COMA marginalizes one action against a
fixed configuration of the others, and a standard difference reward uses a
chosen removal or replacement counterfactual. Neither single-context
construction, by itself, provides the family of coalition contexts needed when
the same retrieval action contributes differently across predecessor
coalitions $S$.

Our distinction is therefore not the Shapley weighting itself, but the
snapshot-grounded construction of comparable coalition evaluations for
parameter-shared LLM retrieval actions, without requiring an agent-wise critic
or a factored value model.

\subsection{Value-Aware Retrieval for Sponsored Search}

Sponsored-search retrieval has evolved from query substitution and dense
retrieval toward generative and value-aware rewriting. Prior work includes
relevance--revenue trade-offs~\citep{relevance_revenue_ads}, diversity-aware RL
rewriting~\citep{clover}, hybrid generative and dense
retrieval~\citep{clover_unity}, contextual
retrieval~\citep{contextual_ads_retrieval}, and value-aware LLM
rewriting~\citep{value_ads_rewriting}. Our focus is complementary: these lines
improve how rewrites are produced or scored, whereas we study within-set credit
after several rewrites are issued jointly and their candidate pools are merged,
deduplicated, ranked, and selected together. In this capacity-limited regime, a
response-level outcome does not by itself determine attribution among parallel
actions, and an action's marginal value need not be invariant to the
surrounding action set.
\section{Method}
\label{sec:method}

We study policy optimization for jointly generated retrieval actions under
set-level feedback. \textbf{CASI} (\emph{Credit Assignment via Set Intervention for Agentic
Retrieval}) constructs alternative coalition evaluations within each
realized action set while holding replayable snapshot components fixed.
Shapley weighting aggregates the resulting contextual marginals, and bounded
zero-sum advantage shaping maps their relative credit to action-level
advantages.

\begin{figure*}[t]
    \centering
    \includegraphics[width=1\textwidth]{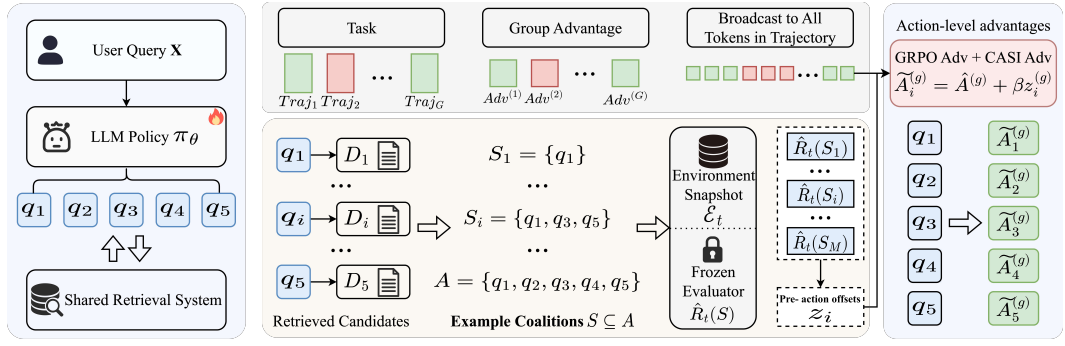}
    \caption{
\textbf{CASI augments the broadcast rollout advantage with action-specific
credit from snapshot-grounded set intervention.}
The base optimizer broadcasts $\hat A^{(g)}$ across the trajectory. Cached
retrievals are recombined into coalitions and scored under the frozen
training-side snapshot $\mathcal{E}_0$, yielding coalition values
$\hat R_0(S)$. Shapley aggregation produces per-action credit, which bounded zero-sum
advantage shaping maps to offsets $z_i^{(g)}$. Tokens of action $q_i$ receive
$\widetilde A_i^{(g)}=\hat A^{(g)}+\beta z_i^{(g)}$; shared tokens retain
$\hat A^{(g)}$.
}
    \label{fig:casi_framework}
\vspace{-15pt}
\end{figure*}

\subsection{Problem Formulation and Coalition-Dependent Credit}
\label{sec:problem}

Given an input $x$, the policy $\pi_\theta$ jointly generates
$A=\{q_1,\ldots,q_m\}\sim\pi_\theta(\cdot\mid x)$.
Each action retrieves candidates $C_{i,t}=\mathcal{R}_t(q_i)$, which are merged
as $C_t(A)=\operatorname{Merge}(C_{1,t},\ldots,C_{m,t})$. We denote the
production-side set utility by
$R_t(A)=\mathcal{V}(C_t(A);\mathcal{E}_t)$, where $t$ indexes the production
state associated with the realized input $x$, and take the production-level
objective to be
\begin{equation}
    J(\theta)
    =
    \mathbb{E}_{x}
    \mathbb{E}_{A\sim\pi_\theta(\cdot\mid x)}
    [R_t(A)].
\end{equation}

The contribution of an action can depend on the surrounding coalition.
For any $S\subseteq A\setminus\{q_i\}$, define the contextual marginal
\begin{equation}
    \Delta^{R}_{i,t}(S)
    =
    R_t(S\cup\{q_i\})-R_t(S).
    \label{eq:marginal}
\end{equation}
In general, $\Delta^{R}_{i,t}(S)$ varies with $S$ because retrieval coverage,
redundancy, and capacity-limited downstream selection depend on which other
actions are present. These production-side counterfactual marginals are not
directly observed from the realized set-level outcome. To construct tractable
within-set comparisons for training, we instead use a frozen surrogate set
function defined below. Appendix~\ref{app:negative} gives examples in which
contextual marginal value can be negative.

\paragraph{Set-level optimization and within-set information.}
For any rollout-level scalar return $F(A)$ with no explicit dependence on
$\theta$, broadcasting the same scalar across the generated sequence gives
the standard score-function estimator
\begin{equation}
\nabla_\theta
\mathbb{E}_{A\sim\pi_\theta(\cdot\mid x)}[F(A)]
=
\mathbb{E}
\left[
    F(A)
    \sum_{u\in\tau(A)}
    \nabla_\theta\log\pi_\theta(u\mid u_{<},x)
\right],
\label{eq:broadcast_pg}
\end{equation}
where $\tau(A)$ is the generated token sequence. Thus, broadcasting $F(A)$
provides a valid score-function estimator for the gradient of the expected
set-level objective. However, $F(A)$ contains only the outcome of the realized
set. It does not determine an attribution whose definition depends on
counterfactual coalition values $F(S)$ for alternative subsets
$S\subseteq A$. Additional full-set rollouts provide more joint outcomes, but
do not directly provide the alternative-subset evaluations needed to construct
such within-set credit for a fixed realized set.

\subsection{Snapshot-Grounded Set Intervention}
\label{sec:snapshot}

CASI constructs alternative coalition evaluations from the same realized
action set under a frozen training-side snapshot $\mathcal{E}_0$;
non-replayable factors remain outside the counterfactual construction. Each
action is retrieved once,
\begin{equation}
    C_i=\mathcal{R}_0(q_i),
    \qquad
    C(S)=\operatorname{Merge}(\{C_i:q_i\in S\}).
\end{equation}
Coalitions are formed by selecting and merging cached per-action results, so
the generated actions, retrievals, and replayable snapshot components remain
fixed while only coalition composition changes. The resulting candidate sets
are scored either by the frozen evaluator for policy optimization or by
auction replay for offline auditing, without new LLM generation or repeated
retrieval.

\paragraph{Offline auction-replay measurement.}
For offline auditing, let $y^{(r)}(S)$ denote the value returned by the
$r$-th execution of the auction-replay path for coalition $S$:
\begin{equation}
    y^{(r)}(S)
    =
    \sum_a
    \mathrm{ecpm}^{(r)}_{\mathrm{auction},S}(a)\,
    \mathrm{send}^{(r)}_S(a),
    \label{eq:replay_value}
\end{equation}
where $\mathrm{ecpm}^{(r)}_{\mathrm{auction},S}(a)$ denotes the auction-stage
eCPM associated with ad $a$ in that replay execution and
$\mathrm{send}^{(r)}_S(a)\in\{0,1\}$ indicates membership in the replay send
set. The send set is an auction-stage outcome rather than final production
exposure. Replay outcomes can vary across executions, so a single
$y^{(r)}(S)$ is an offline replay measurement rather than a time-invariant
production ground truth. Section~\ref{sec:exp_setup} specifies how repeated
executions are aggregated to construct the coalition-audit reference.

\paragraph{Training-time evaluator.}
Auction replay is too expensive to execute for every coalition during policy
optimization, so CASI scores each reconstructed coalition with the frozen
evaluator and denotes the resulting set value by $\hat R_0(S)$. The
corresponding contextual marginal is
\begin{equation}
    \widehat{\Delta}_{i,0}(S)
    =
    \hat R_0(S\cup\{q_i\})-\hat R_0(S).
    \label{eq:surrogate_marginal}
\end{equation}
Coalition values are obtained from cached retrieval results and batched with
the frozen evaluator. CASI changes which values of the fixed surrogate set
function $\hat R_0(\cdot)$ are queried; it does not correct errors in that
evaluator.

\paragraph{Audit boundary.}
\label{sec:fidelity}
The offline evaluator-to-replay analysis compares evaluator-defined coalition
structure and credit with the repeat-aggregated replay reference defined in
Section~\ref{sec:exp_setup}. It does not treat that reference as calibrated
production utility or replay coalition contrasts as production-side contrasts
of $R_t$. Production performance of the learned policy is evaluated separately
in Section~\ref{sec:online}.

\subsection{Coalition-Aware Credit Aggregation}
\label{sec:credit}

We extend the surrogate set function to the empty coalition by the zero-value
convention $\hat R_0(\emptyset)=0$; the evaluator is not queried on the empty
set. Because the contextual marginal $\widehat{\Delta}_{i,0}(S)$ can vary with
the predecessor coalition, a credit rule based on a single context captures
only one part of the induced set function. CASI uses standard Shapley weighting
to aggregate marginals across predecessor coalitions:
\begin{equation}
    c_i^{\mathrm{CASI}}
    =
    \sum_{S\subseteq A\setminus\{q_i\}}
    \frac{|S|!(m-|S|-1)!}{m!}
    \left[
        \hat R_0(S\cup\{q_i\})-\hat R_0(S)
    \right].
    \label{eq:coalition_credit}
\end{equation}
For this extended set function, Shapley efficiency gives
\begin{equation}
    \sum_{i=1}^{m} c_i^{\mathrm{CASI}}
    =
    \hat R_0(A).
\end{equation}
This convention fixes the absolute credit baseline. Holding all non-empty
coalition values fixed, any alternative constant empty-set reference shifts
all Shapley credits equally; the within-rollout centering below removes this
common offset.

\paragraph{Endpoint credit rules.}
We compare against two cheaper rules that use a single predecessor size:
\noindent
\begin{minipage}[t]{0.48\linewidth}
\begin{equation}
    c_i^{\mathrm{Singleton}}
    =
    \hat R_0(\{q_i\})-\hat R_0(\emptyset).
    \label{eq:singleton_credit}
\end{equation}
\end{minipage}
\hfill
\begin{minipage}[t]{0.48\linewidth}
\begin{equation}
    c_i^{\mathrm{LOO}}
    =
    \hat R_0(A)-\hat R_0(A\setminus\{q_i\}).
    \label{eq:loo_credit}
\end{equation}
\end{minipage}
Singleton uses predecessor size $|S|=0$, whereas leave-one-out (LOO) uses
$|S|=m-1$. The following result characterizes the information missing from
these endpoint observations. For a generic set function $v$, let $\phi(v)$
denote its standard Shapley allocation.

\begin{proposition}[Endpoint non-identifiability of the Shapley allocation]
\label{prop:nonident}
Let $m\ge 4$ and let $v:2^{A}\to\mathbb{R}$ be a set function with
$v(\varnothing)=0$. The Shapley allocation $\phi(v)$ is not, in general,
determined by the endpoint observations
\[
\big\{v(\{q_i\})\big\}_{i=1}^{m},\qquad
\big\{v(A\setminus\{q_i\})\big\}_{i=1}^{m},\qquad
v(A).
 \]
\end{proposition}

\begin{proof}
Choose a coalition $T$ with $2\le |T|\le m-2$, which exists for $m\ge4$.
Let $F$ be any set function with $F(\varnothing)=0$ and define $G$ by
$G(T)=F(T)+\delta$ for some $\delta\neq0$, with $G(S)=F(S)$ for every
$S\neq T$. Since $T$ is neither a singleton, an $(m-1)$-coalition, nor the
full set, $F$ and $G$ agree on all listed endpoint observations. For any
$i\in T$, the Shapley term with predecessor set
$S=T\setminus\{q_i\}$ changes by
\[
    \frac{(|T|-1)!(m-|T|)!}{m!}\,\delta\neq0,
\]
while no other term for that $i$ is affected. Hence
$\phi_i(G)\neq\phi_i(F)$ and therefore $\phi(G)\neq\phi(F)$.
\end{proof}

Under additive set values,
\[
    v(S)=\sum_{q_i\in S}a_i,
\]
every contextual marginal equals $a_i$, so Singleton, LOO, and Shapley
coincide. Additional coalition evaluations are therefore redundant in the
additive case. Outside this case, intermediate coalitions can contain
allocation-relevant information that is not determined by endpoint
observations. The proposition concerns information sufficiency for a fixed
set function; whether additional surrogate coalition evaluations are accurate
or useful is evaluated empirically. Alternative aggregation weights,
including Banzhaf weighting, are evaluated separately in
Appendix~\ref{app:credit}.

\subsection{Bounded Zero-Sum Advantage Shaping}
\label{sec:optimization}

For a rollout group $\mathcal{G}$, let $\hat A^{(g)}$ denote the rollout-level
advantage supplied by the base optimizer for rollout $g\in\mathcal{G}$. In our
implementation, this advantage is formed from the full-set surrogate reward
$\hat R_0(A^{(g)})$. Given action credits
$\{c_i^{(g)}\}_{i=1}^{m}$, we first center them within each rollout:
\begin{align}
    d_i^{(g)}
    &=
    c_i^{(g)}
    -
    \frac{1}{m}\sum_{j=1}^{m}c_j^{(g)},
    &
    \sum_{i=1}^{m}d_i^{(g)}
    &=0.
\end{align}
We then normalize the centered credits using a scale shared across the rollout
group:
\begin{align}
    s_d
    &=
    \operatorname{Std}\!\left(\{d_i^{(g)}\}_{i,g}\right),
    &
    u_i^{(g)}
    &=
    \frac{d_i^{(g)}}{s_d+\epsilon},
\end{align}
where $\epsilon>0$ is a numerical stabilizer. Using a common scale avoids
separately normalizing away between-rollout differences in relative-credit
magnitude before projection.

To bound the action-level adjustment while preserving its zero-sum structure,
we project $u^{(g)}$ onto
\[
    \mathcal{Z}_{\kappa}
    =
    \left\{
        z\in\mathbb{R}^{m}:
        \sum_i z_i=0,\;
        \|z\|_{\infty}\le\kappa
    \right\}.
\]
The Euclidean projection has the coordinate form
\begin{equation}
    z_i^{(g)}
    =
    \operatorname{clip}
    \left(
        u_i^{(g)}-\lambda_g,
        -\kappa,
        \kappa
    \right),
    \qquad
    \sum_i z_i^{(g)}=0,
    \label{eq:zero_sum_projection}
\end{equation}
where $\kappa>0$ is the clipping radius and $\lambda_g$ is obtained by
monotone bisection. Appendix~\ref{app:zero_sum} gives the projection
derivation.

The resulting action-level advantage is
\begin{equation}
    \widetilde A_i^{(g)}
    =
    \hat A^{(g)}+\beta z_i^{(g)},
    \label{eq:action_advantage}
\end{equation}
where $\beta\ge0$ controls the strength of credit-aware shaping. Since
$\sum_i z_i^{(g)}=0$,
\begin{equation}
    \frac{1}{m}
    \sum_{i=1}^{m}
    \widetilde A_i^{(g)}
    =
    \hat A^{(g)}.
    \label{eq:zero_sum}
\end{equation}
Thus the shaping term changes relative action-level coefficients while
preserving their within-rollout mean. This does not imply equality with the
broadcast policy gradient, because different action spans multiply different
score functions. Setting $\beta=0$ recovers the Broadcast advantage
assignment.

For structured generations, tokens belonging to action $q_i$ receive
$\widetilde A_i^{(g)}$, while shared reasoning and formatting tokens retain
the base rollout advantage $\hat A^{(g)}$. We treat this as structured
advantage assignment rather than causal token attribution.

\subsection{Computational Cost}
\label{sec:complexity}

All methods use one LLM generation and $m$ cached retrievals; they differ in
the number of frozen-evaluator scores required for credit construction. The
empty coalition is fixed by convention and is not scored. Counting the
full-set score used by the base advantage, Broadcast, Singleton, LOO, and
exact all-coalition aggregation require
\[
    1,\qquad
    m+1,\qquad
    m+1,\qquad
    2^m-1
\]
scores, respectively. Exact CASI and exact Banzhaf aggregation share the
$2^m-1$ all-coalition scoring cost. For $m=5$, the corresponding counts for
Broadcast, Singleton, LOO, and CASI are $1/6/6/31$.
All non-empty coalition scores are computed from cached retrieval results and
batched with the frozen evaluator; the 31 CASI scores are not 31 auction
replays during training. Appendix~\ref{app:complexity} describes permutation
sampling for larger action sets.

\section{Experiments}
\label{sec:experiments}

\subsection{Experimental Setup}
\label{sec:exp_setup}

We study $m=5$ jointly generated retrieval actions. Policy optimization uses
the frozen evaluator $\hat R_0$ defined in Section~\ref{sec:snapshot}, while
offline counterfactual auditing uses the auction-replay procedure defined
there. Historical $V_0$ is used only in the observational audit below.

The training-time evaluator is a cached-filter construction: it processes the
full candidate pool once and obtains coalition values by filtering the cached
downstream result. All primary evaluator-to-replay comparisons use this
evaluator. As a sensitivity analysis, a coalition-rerun evaluator re-executes
downstream ranking for each coalition; because this defines a different
surrogate set function, we do not treat it as an evaluator-error correction.

The offline audit samples $1{,}200$ requests before observing replay outcomes
and enumerates all $31$ non-empty coalitions. Each coalition is replayed three
times; a $50$-request subset receives seven repeats, for $117{,}800$
executions in total. For the coalition audit, we denote the median across
available replay executions by
\[
    y_{\mathrm{audit}}(S)
    =
    \operatorname{median}_{r} y^{(r)}(S),
\]
and use it as the repeat-aggregated coalition reference. Bootstrap inference
resamples requests, and no request is filtered by its full-set replay outcome.

\subsection{Coalition Dependence and Evaluator Fidelity}
\label{sec:exp_offline}

The repeat-aggregated replay landscape exhibits substantial coalition
dependence. On the $1{,}171$ requests with matched training-evaluator values,
corresponding to $11{,}710$ action pairs, $64.4\%$ of pair interactions are
negative, $7.4\%$ positive, and $28.2\%$ exactly zero, with median absolute
interaction $3.0978$. The training-time evaluator agrees with the replay
reference on $92.25\%$ of pair-interaction signs
(95\% CI $[0.9140,0.9295]$).

\begin{figure*}[t]
\centering
\includegraphics[width=0.98\textwidth]{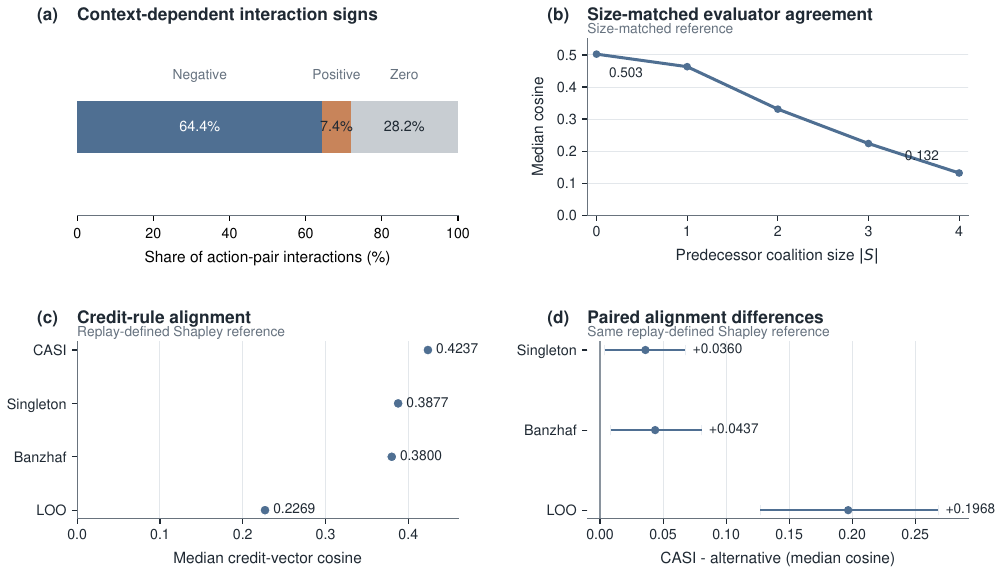}
\caption{
\textbf{Coalition dependence and credit alignment.}
Panels (a)--(b) diagnose context-dependent replay structure and size-resolved
evaluator agreement. Panels (c)--(d) compare credit rules against the same
replay-defined Shapley reference. Because panel (b) uses size-matched
references, its values are not directly comparable with panels (c)--(d).
}
\label{fig:coalition_credit}
\end{figure*}

Full-set outcomes alone do not characterize this structure. Although
$870/1{,}200$ requests satisfy $y_{\mathrm{audit}}(A)>0$, $103/330$ requests
with $y_{\mathrm{audit}}(A)=0$ still contain at least one nonzero proper
subcoalition. Filtering on the observed full-set outcome would therefore
discard requests containing informative coalition structure.

Across nonzero replay insertion edges, the training-time evaluator achieves
$59.25\%$ sign agreement overall. We therefore treat individual
evaluator-defined marginals as imperfect local estimates rather than accurate
action-credit labels.

Size-resolved diagnostics in Table~\ref{tab:size_fidelity} show that
evaluator--replay agreement generally deteriorates as predecessor coalitions
grow. The CASI row instead compares Shapley-aggregated credit with the
replay-defined Shapley reference and is shown only for context.

\begin{table}[!ht]
\centering

\caption{
Size-resolved evaluator--replay agreement under the training-time evaluator.
Rows $|S|=k$ compare request-centered action-credit vectors formed by
averaging marginals over predecessor coalitions of size $k$, using the same
construction for evaluator and replay. The CASI row instead compares
evaluator-defined Shapley credit with the replay-defined Shapley reference
and is not directly comparable across rows.
}
\label{tab:size_fidelity}
\begin{tabular}{lcccc}
\toprule
Credit context & Sign$_{\rm nz}$ & Spearman & Top action & Median cosine \\
\midrule
$|S|=0$ & 0.6424 & 0.3521 & 0.3274 & 0.5027 \\
$|S|=1$ & 0.6356 & 0.3321 & 0.3422 & 0.4635 \\
$|S|=2$ & 0.5994 & 0.2539 & 0.3365 & 0.3310 \\
$|S|=3$ & 0.5683 & 0.1749 & 0.3112 & 0.2236 \\
$|S|=4$ & 0.5381 & 0.1117 & 0.2902 & 0.1316 \\
\midrule
CASI (Shapley) & 0.6178 & 0.3128 & 0.3299 & 0.4237 \\
\bottomrule
\end{tabular}
\end{table}

\subsection{Credit-Rule Alignment with the Replay-Defined Shapley Reference}
\label{sec:exp_credit}

Using $y_{\mathrm{audit}}(\cdot)$, we construct the replay-defined Shapley
reference and center all credit vectors within request. Because CASI also uses
Shapley aggregation, this comparison measures alignment with the chosen
allocation semantics rather than estimator-neutral true action credit.
On $1{,}171$ requests, CASI attains median cosine $0.4237$ versus $0.3877$
for Singleton, a difference in median cosine of $+0.0360$
(95\% CI $[+0.0036,+0.0677]$); Table~\ref{tab:estimator_fidelity}
reports the remaining comparisons.

\begin{table}[!ht]
\centering
\caption{
Alignment with the replay-defined Shapley reference under the training-time
evaluator. Median cosine is the primary metric. Confidence intervals are
request-level paired-bootstrap intervals for the CASI-minus-alternative
difference in median cosine.
}
\label{tab:estimator_fidelity}
\begin{tabular}{lcc}
\toprule
Credit rule & Median cosine $\uparrow$ & CASI $-$ rule (95\% CI) \\
\midrule
CASI      & \textbf{0.4237} & --- \\
Singleton & 0.3877 & $+0.0360\ [0.0036,0.0677]$ \\
Banzhaf   & 0.3800 & $+0.0437\ [0.0083,0.0806]$ \\
LOO       & 0.2269 & $+0.1968\ [0.1272,0.2682]$ \\
\bottomrule
\end{tabular}
\end{table}

Under the coalition-rerun evaluator, the corresponding median-cosine point
estimates are $0.6011$ for CASI and $0.6029$ for Singleton, with CASI slightly
lower on this sensitivity backend.

\paragraph{Additive correlated-error null.}
We test whether the positive CASI-minus-alternative alignment gaps can arise
mechanically from correlated evaluator residuals under an interaction-free
replay reference. Replay singleton values define the additive reference, and
a donor request's complete coalition-level evaluator residual vector is
transferred as a block, preserving the observed within-request residual
covariance. On the stricter $870$-request subset for which the required
normalized credit statistics are defined, the observed CASI--Singleton
alignment gap is $+0.0359$, versus a 95\% null interval of
$[-0.1231,-0.0685]$; for CASI--LOO, the observed gap is $+0.2221$ versus
$[-0.1462,-0.0813]$. This specific additive correlated-residual mechanism
does not reproduce the observed positive gaps. The null does not establish
general evaluator accuracy or imply that the CASI-minus-alternative alignment
gap must increase with interaction strength;
Appendix~\ref{app:estimator_fidelity} gives the full construction.

\subsection{Policy-Level Evaluation}
\label{sec:exp_main}

We next test whether differences in within-set advantage assignment translate
to held-out set-level replay value. Broadcast, Singleton, and CASI use the
same target model and training configuration and are each trained for
$500$ steps with five independent seeds. Final checkpoints are evaluated on
the same $1{,}000$ held-out requests by auction-replaying the generated
five-action set; the primary endpoint is mean full-set replay value.

\begin{table}[!ht]
\centering
\caption{
Held-out policy evaluation on $1{,}000$ common requests; RL entries report
mean $\pm$ standard deviation across five independent training runs.
}
\label{tab:policy_eval}
\begin{tabular}{lcc}
\toprule
Method & Mean replay value $\uparrow$ & Top-1 eCPM $\uparrow$ \\
\midrule
SFT       & 60.32 & 11.25 \\
Broadcast & 43.30$\,\pm\,$4.70 & 11.15$\,\pm\,$0.10 \\
Singleton & 55.00$\,\pm\,$4.00 & 11.68$\,\pm\,$0.09 \\
CASI      & \textbf{66.70$\,\pm\,$3.40} & \textbf{11.93$\,\pm\,$0.07} \\
\bottomrule
\end{tabular}
\end{table}

Across five independent training runs, CASI attains a mean held-out replay
value of $66.70\pm3.40$, compared with $55.00\pm4.00$ for Singleton and
$43.30\pm4.70$ for Broadcast; the SFT initialization obtains $60.32$ on the
same held-out request set. The CASI--Singleton contrast is $+11.70$
(95\% CI $[+6.26,+17.14]$); this run-level inference reflects training-seed
variation on the fixed held-out set, while the endpoint evaluates complete
policy outputs rather than alignment with a Shapley reference.

\subsection{Online A/B Evaluation}
\label{sec:online}

We further evaluate the CASI-trained policy in a randomized seven-day
production A/B test on a large-scale sponsored-search system. The experiment
allocates $10\%$ of production traffic to treatment, covering approximately
$10^8$ observations; the primary metric is advertiser value (Advv). Relative
to the production control, we observe a \textbf{3.51\%} lift in Advv and a
$5.08\%$ increase in search-ad spend. At the retrieval level, complex-intent
coverage rises from $30.86\%$ to $32.58\%$, while first-screen coverage rises
from $16.32\%$ to $16.99\%$. We treat these coverage changes as
mechanism-consistent diagnostics rather than as a causal decomposition of the
online business gain. This randomized comparison evaluates the deployed
CASI-trained policy against production control; it does not isolate CASI from
alternative credit constructions online.

\paragraph{Historical served-value audit.}
\label{sec:exp_served_value}
We separately audit replay against historical served value $V_0$; this is
observational rather than counterfactual
(Appendix~\ref{app:served_audit}). Among the $3{,}540$ requests with
successful replay and matched historical $V_0$, $1{,}438$ are
auction-positive; on this subset, replay has Spearman $\rho=0.346$ with
$V_0$. Holding each generated set fixed, the real-context auction-positive
rate is $40.6\%$, compared with $32.7\%$ under context shuffling, a
real-minus-shuffled difference of $+7.9$ pp (95\% CI $[+6.0,+9.8]$).
Among the $700$ requests positive under both contexts, the difference in
conditional value ranking is $+0.023$ (95\% CI $[-0.061,+0.103]$), providing
no clear evidence of improved ranking under the real context. Thus request
context materially affects auction activation, while this observational audit
does not establish replay as a counterfactual measure of production outcomes.
\section{Conclusion}
\label{sec:conclusion}

We study context-dependent within-set credit for jointly generated retrieval actions under set-level feedback. CASI constructs comparable coalition evaluations through snapshot-grounded set intervention, aggregates contextual marginals with Shapley weighting, and incorporates their relative credit through bounded zero-sum advantage shaping. We show that endpoint observations do not in general determine the Shapley allocation for \(m\ge4\), while additional coalition evaluations are redundant under additive set values. In a five-action sponsored-search setting, CASI achieves the largest median cosine similarity to the replay-defined Shapley reference under the deployed training evaluator, improves held-out auction-replay set value across independent training runs, and yields an observed \(3.51\%\) advertiser-value lift over production control in a randomized A/B test. Backend sensitivity further shows that additional coalition information does not guarantee improved credit alignment under every evaluator. These results support snapshot-grounded set intervention as a practical approach to within-set supervision when retrieval actions interact.

\section*{Use of Large Language Models}
In this work, we used generative AI tools for writing polish of author-written text, including improvements to readability and LaTeX formatting, and for literature retrieval and discovery, including searching for and identifying potentially relevant related work. We did not use generative AI tools to generate synthetic datasets, develop theoretical models or conceptual frameworks, formulate or prove mathematical claims, propose or refine hypotheses, implement methods, or interpret experimental results. All AI-assisted edits and literature suggestions were reviewed by the authors, and AI-suggested references were verified against the original sources through a manual literature survey. We take full responsibility for the final content of this work, including all text, claims, and results.

\bibliographystyle{iclr2027_conference}
\bibliography{iclr2027_conference}

\begin{thebibliography}{18}
\providecommand{\natexlab}[1]{#1}
\providecommand{\url}[1]{\texttt{#1}}
\expandafter\ifx\csname urlstyle\endcsname\relax
  \providecommand{\doi}[1]{doi: #1}\else
  \providecommand{\doi}{doi: \begingroup \urlstyle{rm}\Url}\fi

\bibitem[Foerster et~al.(2018)Foerster, Farquhar, Afouras, Nardelli, and Whiteson]{coma}
Jakob~N. Foerster, Gregory Farquhar, Triantafyllis Afouras, Nantas Nardelli, and Shimon Whiteson.
\newblock Counterfactual multi-agent policy gradients.
\newblock In \emph{Proceedings of the AAAI Conference on Artificial Intelligence}, volume~32, pp.\  2974--2982, 2018.
\newblock \doi{10.1609/aaai.v32i1.11794}.

\bibitem[Grbovic et~al.(2015)Grbovic, Djuric, Radosavljevic, Silvestri, and Bhamidipati]{contextual_ads_retrieval}
Mihajlo Grbovic, Nemanja Djuric, Vladan Radosavljevic, Fabrizio Silvestri, and Narayan Bhamidipati.
\newblock Context- and content-aware embeddings for query rewriting in sponsored search.
\newblock In \emph{Proceedings of the 38th International ACM SIGIR Conference on Research and Development in Information Retrieval}, pp.\  383--392, 2015.
\newblock \doi{10.1145/2766462.2767709}.

\bibitem[Jin et~al.(2025)Jin, Zeng, Yue, Yoon, Ar{\i}k, Wang, Zamani, and Han]{searchr1}
Bowen Jin, Hansi Zeng, Zhenrui Yue, Jinsung Yoon, Sercan~{\"O}. Ar{\i}k, Dong Wang, Hamed Zamani, and Jiawei Han.
\newblock Search-{R1}: Training {LLMs} to reason and leverage search engines with reinforcement learning.
\newblock In \emph{Proceedings of the 2nd Conference on Language Modeling (COLM)}, 2025.

\bibitem[Li et~al.(2021)Li, Kuang, Wang, Liu, Chen, Wu, and Xiao]{shapley_counterfactual}
Jiahui Li, Kun Kuang, Baoxiang Wang, Furui Liu, Long Chen, Fei Wu, and Jun Xiao.
\newblock Shapley counterfactual credits for multi-agent reinforcement learning.
\newblock In \emph{Proceedings of the 27th ACM SIGKDD Conference on Knowledge Discovery and Data Mining}, pp.\  934--942, 2021.
\newblock \doi{10.1145/3447548.3467420}.

\bibitem[Lightman et~al.(2024)Lightman, Kosaraju, Burda, Edwards, Baker, Lee, Leike, Schulman, Sutskever, and Cobbe]{lightman2023verify}
Hunter Lightman, Vineet Kosaraju, Yuri Burda, Harrison Edwards, Bowen Baker, Teddy Lee, Jan Leike, John Schulman, Ilya Sutskever, and Karl Cobbe.
\newblock Let's verify step by step.
\newblock In \emph{International Conference on Learning Representations (ICLR)}, 2024.

\bibitem[Mohankumar et~al.(2021)Mohankumar, Begwani, and Singh]{clover}
Akash~Kumar Mohankumar, Nikit Begwani, and Amit Singh.
\newblock Diversity driven query rewriting in search advertising.
\newblock In \emph{Proceedings of the 27th ACM SIGKDD Conference on Knowledge Discovery and Data Mining}, pp.\  3423--3431, 2021.
\newblock \doi{10.1145/3447548.3467202}.

\bibitem[Mohankumar et~al.(2023)Mohankumar, Dodla, K, and Singh]{clover_unity}
Akash~Kumar Mohankumar, Bhargav Dodla, Gururaj K, and Amit Singh.
\newblock Unified generative \& dense retrieval for query rewriting in sponsored search.
\newblock In \emph{Proceedings of the 32nd ACM International Conference on Information and Knowledge Management}, pp.\  4745--4751, 2023.
\newblock \doi{10.1145/3583780.3615459}.

\bibitem[Radlinski et~al.(2008)Radlinski, Broder, Ciccolo, Gabrilovich, Josifovski, and Riedel]{relevance_revenue_ads}
Filip Radlinski, Andrei~Z. Broder, Peter Ciccolo, Evgeniy Gabrilovich, Vanja Josifovski, and Lance Riedel.
\newblock Optimizing relevance and revenue in ad search: A query substitution approach.
\newblock In \emph{Proceedings of the 31st Annual International ACM SIGIR Conference on Research and Development in Information Retrieval}, pp.\  403--410, 2008.
\newblock \doi{10.1145/1390334.1390404}.

\bibitem[Rashid et~al.(2018)Rashid, Samvelyan, Schroeder, Farquhar, Foerster, and Whiteson]{qmix}
Tabish Rashid, Mikayel Samvelyan, Christian Schroeder, Gregory Farquhar, Jakob Foerster, and Shimon Whiteson.
\newblock {QMIX}: Monotonic value function factorisation for deep multi-agent reinforcement learning.
\newblock In \emph{Proceedings of the 35th International Conference on Machine Learning (ICML)}, volume~80 of \emph{Proceedings of Machine Learning Research}, pp.\  4295--4304, 2018.

\bibitem[Shao et~al.(2024)Shao, Wang, Zhu, Xu, Song, Bi, Zhang, Zhang, Li, Wu, and Guo]{deepseekmath}
Zhihong Shao, Peiyi Wang, Qihao Zhu, Runxin Xu, Junxiao Song, Xiao Bi, Haowei Zhang, Mingchuan Zhang, Y.~K. Li, Y.~Wu, and Daya Guo.
\newblock {DeepSeekMath}: Pushing the limits of mathematical reasoning in open language models.
\newblock \emph{arXiv preprint arXiv:2402.03300}, 2024.

\bibitem[Sunehag et~al.(2018)Sunehag, Lever, Gruslys, Czarnecki, Zambaldi, Jaderberg, Lanctot, Sonnerat, Leibo, Tuyls, and Graepel]{vdn}
Peter Sunehag, Guy Lever, Audrunas Gruslys, Wojciech~Marian Czarnecki, Vinicius Zambaldi, Max Jaderberg, Marc Lanctot, Nicolas Sonnerat, Joel~Z. Leibo, Karl Tuyls, and Thore Graepel.
\newblock Value-decomposition networks for cooperative multi-agent learning based on team reward.
\newblock In \emph{Proceedings of the 17th International Conference on Autonomous Agents and MultiAgent Systems (AAMAS)}, pp.\  2085--2087, 2018.

\bibitem[Swaminathan et~al.(2017)Swaminathan, Krishnamurthy, Agarwal, Dud{\'i}k, Langford, Jose, and Zitouni]{slate_ope}
Adith Swaminathan, Akshay Krishnamurthy, Alekh Agarwal, Miroslav Dud{\'i}k, John Langford, Damien Jose, and Imed Zitouni.
\newblock Off-policy evaluation for slate recommendation.
\newblock In \emph{Advances in Neural Information Processing Systems}, volume~30, pp.\  3632--3642, 2017.

\bibitem[Tan et~al.(2026)Tan, Huang, Wu, Zhang, Zhuang, and Gu]{ragr1}
Zhiwen Tan, Jiaming Huang, Qintong Wu, Hongxuan Zhang, Chenyi Zhuang, and Jinjie Gu.
\newblock {RAG-R1}: Incentivizing the search and reasoning capabilities of {LLMs} through multi-query parallelism.
\newblock In \emph{Proceedings of the AAAI Conference on Artificial Intelligence}, volume~40, pp.\  33187--33195, 2026.
\newblock \doi{10.1609/aaai.v40i39.40603}.

\bibitem[Wang et~al.(2020)Wang, Zhang, Kim, and Gu]{shapley_q}
Jianhong Wang, Yuan Zhang, Tae-Kyun Kim, and Yunjie Gu.
\newblock Shapley {Q}-value: A local reward approach to solve global reward games.
\newblock In \emph{Proceedings of the AAAI Conference on Artificial Intelligence}, volume~34, pp.\  7285--7292, 2020.
\newblock \doi{10.1609/aaai.v34i05.6220}.

\bibitem[Wolpert \& Tumer(2001)Wolpert and Tumer]{difference_rewards}
David~H. Wolpert and Kagan Tumer.
\newblock Optimal payoff functions for members of collectives.
\newblock \emph{Advances in Complex Systems}, 4\penalty0 (2/3):\penalty0 265--279, 2001.
\newblock \doi{10.1142/S0219525901000188}.

\bibitem[Zhang et~al.(2026)Zhang, Chen, Zuo, Li, Wang, Xu, and Zheng]{value_ads_rewriting}
Xiao Zhang, Guanyu Chen, Boyang Zuo, Feng Li, Pengjie Wang, Jian Xu, and Bo~Zheng.
\newblock {VALUE}: Value-aware large language model for query rewriting via weighted trie in sponsored search.
\newblock In \emph{Proceedings of the 32nd ACM SIGKDD Conference on Knowledge Discovery and Data Mining}, pp.\  2584--2593, 2026.

\bibitem[Zhao et~al.(2025)Zhao, Yu, Xu, Singh, Shukla, and Akkiraju]{parallelsearch}
Shu Zhao, Tan Yu, Anbang Xu, Japinder Singh, Aaditya Shukla, and Rama Akkiraju.
\newblock Parallelsearch: Train your {LLMs} to decompose query and search sub-queries in parallel with reinforcement learning.
\newblock \emph{arXiv preprint arXiv:2508.09303}, 2025.

\bibitem[Zheng et~al.(2026)Zheng, Zhang, Yu, Wang, Xing, Dai, Liu, Bao, Huang, Huang, and Yu]{parallelr1}
Tong Zheng, Hongming Zhang, Wenhao Yu, Xiaoyang Wang, He~Xing, Runpeng Dai, Rui Liu, Huiwen Bao, Chengsong Huang, Heng Huang, and Dong Yu.
\newblock Parallel-{R1}: Towards parallel thinking via reinforcement learning.
\newblock In \emph{International Conference on Learning Representations (ICLR)}, 2026.

\end{thebibliography}

\newpage
\appendix
\newpage
\section{Negative and Higher-Order Coalition Effects}
\label{app:negative}

This section formalizes two mechanisms behind coalition-dependent action
credit: competition under finite downstream capacity and misalignment between
the criterion that determines selection and the value credited to selected
candidates. The analysis is deliberately abstract and is not intended as an
exact model of the production serving stack.

Consider the simplified top-$N$ set value
\begin{equation}
    R^{\mathrm{top}}(C)
    =
    \sum_{a\in
    \operatorname{TopN}_{\sigma}(\operatorname{Dedup}(C))}
    v(a),
    \label{eq:topn}
\end{equation}
where $\sigma(a)$ determines selection and $v(a)\geq 0$ is the value credited
to a selected candidate. For an action coalition $S$, let
$C(S)=\operatorname{Merge}(\{C_i:q_i\in S\})$ and define
\begin{equation}
    \Delta_i^{\mathrm{top}}(S)
    =
    R^{\mathrm{top}}(C(S\cup\{q_i\}))
    -R^{\mathrm{top}}(C(S)).
\end{equation}
The superscript distinguishes this static abstraction from the production
utility $R_t$ in Section~\ref{sec:problem}.

\paragraph{Selection--value alignment excludes negative insertion marginals.}
Assume that $\sigma$, together with a fixed tie-breaking rule, induces a total
order $\succ_\sigma$ that is co-monotone with credited value:
\begin{equation}
    a\succ_\sigma b \quad\Longrightarrow\quad v(a)\geq v(b).
\end{equation}
Then $R^{\mathrm{top}}$ is non-decreasing under candidate insertion, and hence
$\Delta_i^{\mathrm{top}}(S)\geq 0$ for every action $i$ and coalition
$S\subseteq A\setminus\{q_i\}$.

\paragraph{Proof.}
It suffices to consider inserting one candidate $a$ into a deduplicated pool.
If $a$ is a duplicate, the value is unchanged. If fewer than $N$ positions are
occupied, the change is $v(a)\geq0$. Otherwise, $a$ enters the selected set
only by displacing its lowest-ranked selected member $b$. This requires
$a\succ_\sigma b$, so co-monotonicity gives $v(a)\geq v(b)$ and the value
change is $v(a)-v(b)\geq0$. Inserting all candidates contributed by one action
as a sequence of such operations proves the result. \hfill$\square$

\paragraph{Finite capacity still induces higher-order interaction.}
This result rules out negative first-order insertion marginals under its
alignment assumptions, but it does not imply additivity. Once the available
positions are saturated, an action can have substantial value in a small
coalition and little or no additional value in a larger coalition. For
example, with $N=1$ and two non-duplicate candidates of values $10$ and $9$,
aligned selection gives
$R^{\mathrm{top}}(\{a\})=10$,
$R^{\mathrm{top}}(\{b\})=9$, and
$R^{\mathrm{top}}(\{a,b\})=10$, so the second-order interaction is
$10-10-9+0=-9$.

\paragraph{Relation to the auction-replay reference.}
Equation~\eqref{eq:topn} is a sufficient-condition abstraction rather than a
literal reconstruction of the auction replay in Eq.~\eqref{eq:replay_value}. In the
replay system, coalition insertion can change the winner/send set through
ranking, auction, capacity, and other replayable constraints. Negative
first-order marginals are therefore possible whenever the induced selection
logic is not perfectly aligned with the value credited by the replay
reference. Even under perfect alignment, finite capacity is sufficient to
produce redundancy and negative higher-order interactions. We consequently do
not attribute an observed negative marginal or interaction to a unique serving
mechanism without an additional intervention.

\section{Coalition Credit: Properties and Variants}
\label{app:credit}

Fix one generated action set $A=\{q_1,\ldots,q_m\}$ and the frozen surrogate
set function $v(S)=\hat R_0(S)$, with $v(\emptyset)=0$. For action $q_i$, write
\begin{equation}
    \widehat\Delta_i(S)
    =
    v(S\cup\{q_i\})-v(S),
    \qquad
    S\subseteq A\setminus\{q_i\}.
\end{equation}
CASI uses the Shapley aggregation
\begin{equation}
    c_i^{\mathrm{CASI}}
    =
    \sum_{S\subseteq A\setminus\{q_i\}}
    \frac{|S|!(m-|S|-1)!}{m!}\,
    \widehat\Delta_i(S).
    \label{eq:app_casi_credit}
\end{equation}

\paragraph{Equal total mass over predecessor sizes.}
For $k\in\{0,\ldots,m-1\}$, define the average marginal at predecessor size
$k$ as
\begin{equation}
    \bar\Delta_i^{(k)}
    =
    \frac{1}{\binom{m-1}{k}}
    \sum_{\substack{S\subseteq A\setminus\{q_i\}\\|S|=k}}
    \widehat\Delta_i(S).
    \label{eq:size_avg_marginal}
\end{equation}
For every subset of size $k$, the Shapley weight is
$k!(m-k-1)!/m!$. Since there are $\binom{m-1}{k}$ such subsets,
\begin{equation}
    \binom{m-1}{k}
    \frac{k!(m-k-1)!}{m!}
    =
    \frac{1}{m}.
\end{equation}
Therefore
\begin{equation}
    c_i^{\mathrm{CASI}}
    =
    \frac{1}{m}
    \sum_{k=0}^{m-1}\bar\Delta_i^{(k)}.
    \label{eq:casi_size_decomposition}
\end{equation}
Thus CASI gives equal total weight to each predecessor-size regime, while
averaging uniformly over coalitions within a regime. This identity is useful
for interpreting size-resolved credit diagnostics.

\paragraph{Efficiency.}
The Shapley permutation representation writes $c_i^{\mathrm{CASI}}$ as the
expected marginal contribution of $q_i$ along a uniformly random permutation
of the $m$ actions. For any fixed permutation, summing the $m$ successive
marginals telescopes from $v(\emptyset)$ to $v(A)$. Averaging over
permutations therefore gives
\begin{equation}
    \sum_{i=1}^{m}c_i^{\mathrm{CASI}}
    =v(A)-v(\emptyset)
    =\hat R_0(A).
\end{equation}
The optimizer below uses only relative within-set credit, but this identity
makes the allocation internally coherent before centering.

\paragraph{Single-context and alternative aggregation rules.}
The two endpoint rules used in our controlled comparisons are
\begin{align}
    c_i^{\mathrm{Singleton}}
    &=v(\{q_i\})-v(\emptyset),\\
    c_i^{\mathrm{LOO}}
    &=v(A)-v(A\setminus\{q_i\}).
\end{align}
Singleton uses only predecessor size $0$, whereas leave-one-out (LOO) uses
only predecessor size $m-1$. Banzhaf aggregation instead assigns equal weight
to individual predecessor subsets,
\begin{equation}
    c_i^{\mathrm{Banzhaf}}
    =
    \frac{1}{2^{m-1}}
    \sum_{S\subseteq A\setminus\{q_i\}}
    \widehat\Delta_i(S).
\end{equation}
The zero-sum optimizer is agnostic to which of these rules produces the
per-action credit vector; CASI uses Eq.~\eqref{eq:app_casi_credit} by default.

\paragraph{Additive case.}
Suppose the frozen evaluator is additive,
\begin{equation}
    v(S)=\sum_{q_i\in S}u_i.
\end{equation}
Then $\widehat\Delta_i(S)=u_i$ for every predecessor coalition $S$. Hence
\begin{equation}
    c_i^{\mathrm{CASI}}
    =c_i^{\mathrm{Singleton}}
    =c_i^{\mathrm{LOO}}
    =c_i^{\mathrm{Banzhaf}}
    =u_i.
\end{equation}
After within-set centering, all four rules therefore produce the same relative
credit $u_i-\bar u$. In this additive case, extra coalition evaluations are
redundant for credit allocation; outside it, they can access coalition values
that endpoint rules do not observe.

\paragraph{Scope of the attribution.}
CASI attributes downstream value within a fixed generated action set. It does
not attempt to identify the generation-time causal effect by which an early
autoregressive action may change which later actions are emitted. Likewise,
within-set centering is not interpreted as an action-independent
score-function baseline: the generated actions are jointly downstream of the
policy trajectory. The resulting update is therefore treated as structured
advantage shaping rather than as a new unbiased policy-gradient identity.

\section{Bounded Zero-Sum Advantage Shaping}
\label{app:zero_sum}

For rollout $g\in\mathcal{G}$, let $c_i^{(g)}$ denote the chosen per-action
credit and let $\hat A^{(g)}$ be the rollout-level advantage supplied by the
base optimizer. CASI first centers the credit,
\begin{equation}
    d_i^{(g)}
    =
    c_i^{(g)}-\frac{1}{m}\sum_{j=1}^{m}c_j^{(g)},
    \qquad
    \sum_i d_i^{(g)}=0,
\end{equation}
and normalizes it with the group-level scale
\begin{equation}
    s_d
    =
    \operatorname{Std}\!\left(\{d_i^{(g)}\}_{i,g}\right),
    \qquad
    u_i^{(g)}
    =
    \frac{d_i^{(g)}}{s_d+\epsilon}.
\end{equation}
In exact arithmetic, $\sum_i u_i^{(g)}=0$. Independent element-wise clipping,
however, does not preserve this identity in general:
\begin{equation}
    \sum_i \operatorname{clip}(u_i^{(g)},-\kappa,\kappa)
    \neq 0.
\end{equation}
We therefore use the bounded zero-sum projection from
Section~\ref{sec:optimization}.

\paragraph{Bounded zero-sum projection.}
For a normalized credit vector $u\in\mathbb{R}^m$ and clipping radius
$\kappa>0$, define
\begin{equation}
    z^\star
    =
    \arg\min_{z\in\mathbb{R}^m}
    \frac{1}{2}\|z-u\|_2^2
    \quad\text{s.t.}\quad
    \mathbf{1}^\top z=0,
    \qquad
    -\kappa\le z_i\le \kappa.
    \label{eq:bounded_projection_problem}
\end{equation}
The feasible set is non-empty, closed, and convex, and the objective is
strictly convex, so the projected vector $z^\star$ is unique. The KKT
conditions give a scalar equality multiplier $\lambda$ for which the
box-constrained coordinate minimizer is
\begin{equation}
    z_i^\star
    =
    \operatorname{clip}(u_i-\lambda,-\kappa,\kappa).
    \label{eq:bounded_projection_solution}
\end{equation}
It remains to choose $\lambda$ so that $\sum_i z_i^\star=0$. Define
\begin{equation}
    h(\lambda)
    =
    \sum_i \operatorname{clip}(u_i-\lambda,-\kappa,\kappa).
\end{equation}
The function $h$ is continuous and non-increasing, with
\begin{equation}
    h(\min_i u_i-\kappa)=m\kappa,
    \qquad
    h(\max_i u_i+\kappa)=-m\kappa.
\end{equation}
A root therefore exists by the intermediate value theorem and can be found by
monotone bisection on
\begin{equation}
    \lambda\in
    [\min_i u_i-\kappa,\;\max_i u_i+\kappa].
\end{equation}
The scalar root need not be unique if $h$ is flat at zero, but every such root
yields the same unique projected vector $z^\star$. For the small $m$ used in
our experiments, this scalar solve is negligible relative to evaluator
scoring.

\paragraph{Mean-preserving action advantages.}
CASI forms
\begin{equation}
    \widetilde A_i^{(g)}
    =
    \hat A^{(g)}+\beta z_i^{(g)}.
\end{equation}
Since $\sum_i z_i^{(g)}=0$,
\begin{equation}
    \frac{1}{m}\sum_i \widetilde A_i^{(g)}
    =
    \hat A^{(g)}.
\end{equation}
Thus shaping changes only the relative action-level coefficients within the
rollout. This does \emph{not} imply that the shaped policy gradient equals the
broadcast policy gradient, because different action spans multiply different
score functions. Setting $\beta=0$ exactly recovers ordinary broadcasting.
For structured generations, tokens belonging to action $q_i$ receive
$\widetilde A_i^{(g)}$, while shared reasoning and formatting tokens retain
the base rollout advantage $\hat A^{(g)}$.

\section{Permutation Sampling for Larger Action Sets}
\label{app:complexity}

For the default $m=5$, exact CASI evaluates the $2^5-1=31$ non-empty
coalitions. These are frozen-evaluator scores computed from cached action-level
retrieval results, not 31 downstream auction replays during training. Exact
enumeration is therefore practical at this scale and introduces no
coalition-sampling error.

For larger action sets, the same Shapley credit admits a permutation estimator.
Let $\pi^{(r)}$ be the $r$-th uniformly sampled permutation of the $m$ actions,
and let $S_i^{(r)}$ contain the actions preceding $q_i$ in that permutation.
Then
\begin{equation}
    \hat c_i
    =
    \frac{1}{M}
    \sum_{r=1}^{M}
    \left[
        \hat R_0(S_i^{(r)}\cup\{q_i\})
        -
        \hat R_0(S_i^{(r)})
    \right].
    \label{eq:permutation_estimator}
\end{equation}
Under uniform permutation sampling,
$\mathbb{E}[\hat c_i]=c_i^{\mathrm{CASI}}$. A naive implementation requires
$O(Mm)$ prefix evaluations, although repeated subsets across sampled
permutations can be deduplicated and scored once. This provides a direct
accuracy--cost trade-off when exact enumeration of all $2^m-1$ non-empty
coalitions is no longer practical.

\section{Offline Replay Landscape and Differential-Fidelity Audit}
\label{app:audit}

This section gives the construction and diagnostics behind the offline
counterfactual reference used in Sections~\ref{sec:exp_setup} and
\ref{sec:exp_offline}. The goal is to audit whether the frozen training
evaluator preserves coalition differences under the same reconstructed action
sets. It is distinct from the historical served-value audit in
Appendix~\ref{app:served_audit}.

\paragraph{Replay landscape.}
For each sampled request, the SFT checkpoint generates one five-action set
$A=\{q_1,\ldots,q_5\}$. We enumerate all $2^5-1=31$ non-empty coalitions and
replay each coalition through the downstream auction-replay path. For replay
execution $r$, the measured coalition value is
\begin{equation}
    y^{(r)}(S)
    =
    \sum_a
    \mathrm{ecpm}^{(r)}_{\mathrm{auction},S}(a)\,
    \mathrm{send}^{(r)}_S(a),
\end{equation}
where $\mathrm{send}^{(r)}_S$ is the replay winner/send indicator rather than
final production exposure. The audit contains $1{,}200$ requests. Every
coalition is replayed three times; a subset of $50$ requests is extended to
seven repeats, for a total of $117{,}800$ full-chain replay executions. The
repeat-aggregated coalition reference is
\[
    y_{\mathrm{audit}}(S)
    =
    \operatorname{median}_{r} y^{(r)}(S).
\]
No request is admitted or rejected according to its observed full-set replay
outcome.

This last design choice matters empirically. The full set is replay-positive
for $870/1{,}200$ requests, while $103/330$ requests with zero full-set value
still contain at least one nonzero proper subcoalition. Conditioning sample
admission on $y_{\mathrm{audit}}(A)>0$ would therefore remove requests that
contain informative coalition effects.

\paragraph{Training and sensitivity evaluator backends.}
The primary evaluator is \texttt{pr\_filter}, which matches the backend used
during RL: the full candidate pool is processed once and coalition values are
obtained by filtering the corresponding cached result. We also report
\texttt{pr\_rerun} as a sensitivity backend, where downstream ranking is
re-executed separately for each coalition. The two backends are not
interchangeable. Their insertion-edge sign agreement is only $0.6152$
(95\% CI $[0.6097,0.6209]$), although they agree much more strongly on
pair-interaction signs ($0.9697$). All primary action-credit comparisons are
therefore reported conditional on \texttt{pr\_filter}, the evaluator consumed
during RL.

\paragraph{Pair interactions.}
For actions $i$ and $j$, we use the second-order interaction
\begin{equation}
    I_{ij}
    =
    v(\{i,j\})-v(\{i\})-v(\{j\})+v(\emptyset).
\end{equation}
On the $1{,}171$ requests covered by \texttt{pr\_filter}, the replay
landscape contains $11{,}710$ action pairs. Replay interactions are negative
for $64.4\%$ of pairs, positive for $7.4\%$, and exactly zero for $28.2\%$,
with median absolute interaction $3.0978$. The \texttt{pr\_filter} evaluator
agrees with replay on $92.25\%$ of interaction signs
(95\% CI $[0.9140,0.9295]$).

\paragraph{Insertion-edge fidelity.}
The credit rules operate on neighboring-coalition differences
\begin{equation}
    \Delta_i^{\mathrm{audit}}(S)
    =
    y_{\mathrm{audit}}(S\cup\{q_i\})-y_{\mathrm{audit}}(S),
    \qquad
    \widehat\Delta_i(S)
    =
    \hat R_0(S\cup\{q_i\})-\hat R_0(S).
\end{equation}
Because exact-zero replay edges do not have a meaningful positive/negative
sign, the primary sign-agreement analysis is restricted to
$\Delta_i^{\mathrm{audit}}(S)\neq0$. Confidence intervals resample requests rather than
individual edges.

\begin{table}[h]
\centering
\caption{Insertion-edge sign agreement with replay, stratified by predecessor
size. \texttt{pr\_filter} is the training backend; \texttt{pr\_rerun} is a
sensitivity backend.}
\label{tab:app_edge_fidelity}
\begin{tabular}{lcc}
\toprule
Predecessor size & \texttt{pr\_filter} & \texttt{pr\_rerun} \\
\midrule
All nonzero replay edges & 0.5925 & 0.6137 \\
$|S|=0$ & 0.9839 & 1.0000 \\
$|S|=1$ & 0.6257 & 0.6572 \\
$|S|=2$ & 0.5617 & 0.5819 \\
$|S|=3$ & 0.5291 & 0.5485 \\
$|S|=4$ & 0.5375 & 0.5348 \\
\bottomrule
\end{tabular}
\end{table}

The near-perfect $|S|=0$ sign agreement is a trivial boundary case because an
insertion from the empty coalition is non-negative under the replay-value
construction. For nontrivial predecessor sizes, local marginal fidelity is
substantially lower and is weakest near the grand-coalition regime used by
LOO.

\paragraph{Replay variability and magnitude-conditioned reliability.}
Replay variability is estimated by repeatedly evaluating the \emph{same
insertion edge} under the same request conditions. On the deeply repeated
subset, median aggregation has higher split-half sign self-consistency than
mean aggregation: $0.7747$ versus $0.7214$ when restricted to nonzero replay
edges. Increasing the number of repeats from one to seven reduces the median
replay-noise scale by roughly $2.6\times$, but does not improve evaluator-vs.-
replay sign agreement. This indicates that evaluator error, rather than replay
randomness alone, is the main bottleneck in local-credit fidelity.

Reliability also depends strongly on edge magnitude. Sign agreement is near
chance for the smallest replay differences and rises for larger differences.
However, the replay reference itself has split-half sign self-consistency of
only about $0.775$ on nonzero edges. We therefore do not define a universal
absolute threshold $\tau$ at which signs are declared reliable. Instead, the
paper reports magnitude-conditioned fidelity together with replay
self-consistency as a reference diagnostic.

\section{Size-Resolved Credit-Fidelity Diagnostics}
\label{app:size_fidelity}

To distinguish local edge quality from complete action-level credit, we group
marginals by predecessor size. For request-level set function $v$ and action
$i$, define
\begin{equation}
    c_i^{(k)}
    =
    \frac{1}{\binom{m-1}{k}}
    \sum_{\substack{S\subseteq A\setminus\{q_i\}\\|S|=k}}
    \bigl[v(S\cup\{q_i\})-v(S)\bigr],
    \qquad k=0,\ldots,m-1.
\end{equation}
Credits are centered within request before comparison. For each $k$, the
evaluator credit is compared with replay credit defined at the \emph{same}
predecessor size. The final Shapley/CASI row instead compares complete
evaluator-defined Shapley credit with the replay-defined Shapley reference and
is included only for context. The size-matched rows and the Shapley row
therefore answer different diagnostic questions and should not be read as a
single cross-row estimator ranking.

\begin{table}[h]
\centering
\caption{Size-resolved credit fidelity for the training-time
\texttt{pr\_filter} evaluator. The $|S|=k$ rows are size-matched diagnostics;
the Shapley/CASI row compares against the replay-defined Shapley reference and
is not directly comparable with the size-matched rows.}
\label{tab:app_size_fidelity}
\begin{tabular}{lccccc}
\toprule
Context & Sign$_{\rm all}$ & Sign$_{\rm nz}$ & Spearman & Top action & Median cosine \\
\midrule
$|S|=0$ & 0.6056 & 0.6424 & 0.3521 & 0.3274 & 0.5027 \\
$|S|=1$ & 0.6135 & 0.6356 & 0.3321 & 0.3422 & 0.4635 \\
$|S|=2$ & 0.6003 & 0.5994 & 0.2539 & 0.3365 & 0.3310 \\
$|S|=3$ & 0.5862 & 0.5683 & 0.1749 & 0.3112 & 0.2236 \\
$|S|=4$ & 0.5687 & 0.5381 & 0.1117 & 0.2902 & 0.1316 \\
\midrule
Shapley / CASI & 0.6029 & 0.6178 & 0.3128 & 0.3299 & 0.4237 \\
\bottomrule
\end{tabular}
\end{table}

Among the size-matched rows, direction, rank, and vector-angle fidelity
generally deteriorate as the predecessor coalition grows; top-action agreement
shows the same overall pattern but is not strictly monotone. The Shapley/CASI
row uses a different replay reference and should not be interpreted as a
cross-row ranking.

\section{Common-Reference Credit-Rule Comparison}
\label{app:estimator_fidelity}

The size-resolved analysis above compares each estimator with a replay target
defined at the same predecessor size. To compare complete credit rules against
one common reference, we instead use centered replay-defined Shapley credit.
This reference is intentionally aligned with CASI's allocation rule; the
experiment measures alignment with that replay-defined allocation rather than
establishing an estimator-neutral notion of true action credit.

\paragraph{Estimator definitions and metrics.}
We compare CASI, Singleton, leave-one-out (LOO), and Banzhaf credits. All raw
credit vectors are centered within request before evaluation. We report pooled
sign agreement, sign agreement on nonzero target components, pooled Spearman
correlation, top-action agreement, and the median cosine between normalized
five-dimensional credit vectors. The latter is the primary vector-level
metric.

\begin{table}[h]
\centering
\caption{Full credit-fidelity metrics against centered replay-defined Shapley
credit under the training-time \texttt{pr\_filter} backend.}
\label{tab:app_estimator_full}
\begin{tabular}{lccccc}
\toprule
Rule & Sign$_{\rm all}$ & Sign$_{\rm nz}$ & Spearman & Top action & Median cosine \\
\midrule
CASI & 0.6029 & 0.6178 & 0.3128 & 0.3299 & \textbf{0.4237} \\
Singleton & 0.5947 & 0.6193 & 0.3135 & 0.3205 & 0.3877 \\
Banzhaf & 0.6060 & 0.6149 & 0.3014 & 0.3371 & 0.3800 \\
LOO & 0.5875 & 0.5693 & 0.1761 & 0.3226 & 0.2269 \\
\bottomrule
\end{tabular}
\end{table}

On the primary metric, CASI exceeds Singleton by $+0.0360$
(95\% CI $[+0.0036,+0.0677]$), Banzhaf by $+0.0437$
($[+0.0083,+0.0806]$), and LOO by $+0.1968$
($[+0.1272,+0.2682]$). These confidence intervals are request-level paired
bootstrap intervals.

\paragraph{Sensitivity to evaluator backend.}
Using the alternative \texttt{pr\_rerun} backend changes the comparison
substantially. Its median-cosine values are $0.6011$ for CASI, $0.6029$ for
Singleton, $0.5984$ for Banzhaf, and $0.2804$ for LOO. CASI is slightly lower
in point estimate than Singleton on this sensitivity backend. This marks an
applicability boundary: under \texttt{pr\_rerun}, additional coalition
queries do not improve the primary metric relative to Singleton. It also limits
evaluator-independent generalization; CASI changes which coalition values are
queried from a fixed evaluator rather than correcting evaluator error.

\paragraph{Additive correlated-error null.}
To test whether CASI's gain can arise merely from averaging correlated
evaluator noise, we construct an interaction-free null while preserving the
observed within-request error covariance. For each request $q$, replay
singleton values define
\begin{equation}
    u_i^q=y_{\mathrm{audit},q}(\{i\}),
    \qquad
    v_{\rm null}^q(S)=\sum_{i\in S}u_i^q.
\end{equation}
We then sample a donor request $q'$ and transfer its complete coalition-level
evaluator residual vector
\begin{equation}
    r^{q'}(S)
    =
    \hat R_{0,q'}(S)-y_{\mathrm{audit},q'}(S)
\end{equation}
as a block:
\begin{equation}
    \hat R_{0,\rm null}^{q}(S)
    =
    v_{\rm null}^{q}(S)+r^{q'}(S).
\end{equation}
This preserves the observed correlation geometry among the donor's coalition
errors rather than resampling individual edges independently.

The null analysis uses the stricter subset of $870$ requests for which the
required normalized credit statistics are all defined and repeats the
construction for $1{,}000$ null worlds.

\begin{table}[h]
\centering
\caption{Additive correlated-error null. Observed gaps are computed on the
same $870$-request subset.}
\label{tab:app_null}
\begin{tabular}{lccc}
\toprule
Statistic & Observed & Null median & Null 95\% interval \\
\midrule
$\Delta\cos_{\rm med}(\mathrm{CASI}-\mathrm{LOO})$
& +0.2221 & -0.1105 & $[-0.1462,-0.0813]$ \\
$\Delta\cos_{\rm med}(\mathrm{CASI}-\mathrm{Singleton})$
& +0.0359 & -0.0932 & $[-0.1231,-0.0685]$ \\
$\rho(\Gamma_{\rm ctx}^{\rm norm},g^{\rm CASI-LOO})$
& -0.0270 & +0.0052 & $[-0.0630,+0.0758]$ \\
$\rho(\Gamma_{\rm ctx}^{\rm norm},g^{\rm CASI-Singleton})$
& -0.0699 & +0.0092 & $[-0.0575,+0.0756]$ \\
\bottomrule
\end{tabular}
\end{table}

The observed overall CASI gaps are not reproduced by the additive null with
the measured correlated-error geometry. By contrast, the interaction-strength
trend does not provide the same support: the CASI--LOO trend is fully
compatible with the null, while the CASI--Singleton trend is weak and negative.
We therefore do not claim that CASI's advantage increases monotonically with
interaction strength.

\section{Historical Served-Value Observational Audit}
\label{app:served_audit}

This audit addresses a different question from the counterfactual fidelity
analysis above. Historical downstream value $V_0$ is observed on served
traffic, whereas the evaluated checkpoint generates its own action set.
Consequently, the audit does not provide arbitrary counterfactual values
$V_0(S)$ and is not used as action-level ground truth.

\paragraph{Real-context replay.}
For $4{,}486$ target queries, the SFT checkpoint generates eight rollouts and a
deterministic request-level hash selects one fixed action set. The full set is
replayed using that request's original user and geographic context. Successful
replay and matched historical $V_0$ are both available for $3{,}540$ requests.
Because zero-auction outcomes are a normal structural outcome of the replay
path, we report value ranking on the $1{,}438$ auction-positive requests and
analyze auction activation separately. On this subset,
\begin{equation}
    \rho(y_{\rm replay},V_0)=0.346.
\end{equation}
We interpret this as a moderate observational association rather than
counterfactual validation.

\paragraph{Matched context-shuffling placebo.}
To isolate the role of request context, we hold every query and generated
action set fixed and assign each request another request's real context using a
derangement. All $3{,}540$ requests are replayed under both real and shuffled
contexts.

\begin{table}[h]
\centering
\caption{Matched real-context versus shuffled-context replay.}
\label{tab:app_context_shuffle}
\begin{tabular}{lcc}
\toprule
Metric & Real context & Shuffled context \\
\midrule
Auction-positive rate & 40.6\% & 32.7\% \\
$\rho(y,V_0)$ on both-positive requests ($n=700$) & 0.374 & 0.351 \\
\bottomrule
\end{tabular}
\end{table}

The activation difference is $+7.9$ percentage points
(95\% CI $[+6.0,+9.8]$ pp; McNemar $p\approx6.9\times10^{-16}$).
Thus request context materially affects whether replay reaches a positive
auction outcome. Conditional on both replays being positive, however, the
paired difference in rank correlation is only $+0.023$
(95\% CI $[-0.061,+0.103]$), so we do not detect an improvement in
cross-request value ranking from the correct context.

\paragraph{Temporal self-consistency benchmark.}
As a scale reference, we compute pairwise Spearman correlation of historical
$V_0$ for the same queries across four snapshot dates. The six pairwise
correlations range from $0.383$ to $0.455$, with mean $0.428$. We use this only
as a temporal self-consistency benchmark. It is not an upper bound on
replay--$V_0$ association because the two comparisons have different
estimands.

Overall, the historical audit supports two limited statements: replay value
retains moderate observational association with historical business value, and
request context strongly affects auction activation. It does not establish a
validated mapping from counterfactual replay $y_{\mathrm{audit}}(S)$ to production
counterfactual utility.

\section{Implementation Validation for Exact Coalition Shaping}
\label{app:implementation}

We separately validate that the exact coalition path and the bounded zero-sum
transformation in Appendix~\ref{app:zero_sum} are realized by the training
implementation. This section is an implementation check rather than a
policy-quality evaluation.

\paragraph{Exact coalition path.}
A real-training smoke test executes the exact $31$-coalition path for three
steps with the same \texttt{pr\_filter} evaluator used by training. All three
steps complete successfully.

\paragraph{Why independent clipping is insufficient.}
Before the bounded projection was introduced, the implementation applied
independent element-wise clipping to normalized centered credits. In the smoke
test, the first two steps have zero clip hits and maintain the zero-sum
residual at approximately $10^{-6}$--$10^{-5}$. At the third step, only
$2.86\%$ of action credits hit the clip bound, yet
\begin{equation}
    \left|\sum_i z_i\right|_{\max}
\end{equation}
jumps to $15.0$. This directly confirms that independent clipping can severely
violate the zero-sum identity even when the clip-hit rate is small.

The final implementation instead solves the bounded zero-sum projection in
Eq.~\eqref{eq:bounded_projection_solution}. In subsequent long-running
Singleton and CASI training diagnostics, the maximum zero-sum residual remains
on the order of $2$--$3\times10^{-5}$ even near clip-hit peaks. These checks
support implementation correctness only; they are not used as evidence of
policy improvement.

\end{document}